\documentclass[11pt]{article}

\usepackage{amsmath,amsthm,amssymb,amscd,amstext,amsfonts}
\usepackage{color,verbatim,graphicx,fullpage,url}
\usepackage{xcolor}
\usepackage{algorithm}
\usepackage{algorithmic}
\usepackage{nicefrac}
\usepackage{fullpage,tikz,enumitem}
\usepackage{tikz-cd}
\usepackage{bm}
\usepackage{physics}
\usepackage{todonotes}
\usepackage{mathtools}
\usepackage{thm-restate}
\usepackage[numbers]{natbib}

\usepackage[breaklinks=true, linktocpage=true,]{hyperref}
\hypersetup{
    colorlinks = true,
    linkcolor={purple},
    urlcolor={purple},
    citecolor={blue!80!black}
}

\usepackage[nameinlink, noabbrev, capitalize]{cleveref}

\usepackage[utf8]{inputenc}
 
\Crefname{enumi}{Property}{Properties}

\theoremstyle{plain}

\newtheorem{thm}{Theorem}[section]
\newtheorem{theorem}[thm]{Theorem}

\crefname{atheorem}{Theorem}{Theorems}
\Crefname{atheorem}{Theorem}{Theorems}

\newtheorem{lemma}[thm]{Lemma}

\newtheorem{definition}[thm]{Definition}

\newtheorem{question}[thm]{Question}
\newtheorem{claim}[thm]{Claim}
\newtheorem*{claim*}{Claim}

\newtheorem{remark}[thm]{Remark}

\RenewCommandCopy{\theHtheorem}{\thetheorem}

\renewcommand\norm[1]{\left|\!\left|#1\right|\!\right|}

\newcommand{\set}[1]{\{#1\}}
\newcommand{\Set}[1]{\left\{#1\right\}}

\newcommand{\inp}[2]{{\left\langle #1,#2 \right\rangle}}            

\newcommand{\R}{\mathbb{R}}
\newcommand{\bbS}{\mathbb{S}}

\newcommand{\cX}{\mathcal X}

\newcommand{\eps}{\varepsilon}

\usepackage{ifthen}
\newcommand{\agnorm}[2][]{
	\ifthenelse{\equal{#2}{}}{
		\widetilde{\gamma}_2^{#1}
	}{
		\widetilde{\gamma}_2^{#1}(#2)
	}
}

\DeclareFontFamily{U}{mathx}{}
\DeclareFontShape{U}{mathx}{m}{n}{<-> mathx10}{}
\DeclareSymbolFont{mathx}{U}{mathx}{m}{n}
\DeclareMathAccent{\widecheck}{0}{mathx}{"71}

\newcommand{\cC}{\mathcal{C}}
\newcommand{\cH}{\mathcal{H}}
\newcommand{\cA}{\mathcal{A}}

\newcommand{\cE}{\mathcal{E}}

\newcommand{\cP}{\mathcal{P}}
\newcommand{\cU}{\mathcal{U}}

\DeclareMathOperator{\VCdim}{\textsc{vc}}

\DeclareMathOperator{\LR}{\textsc{lr}}

\DeclareMathOperator{\sign}{sign}

\renewcommand{\cH}{\mathcal{H}}

\renewcommand{\Pr}{\mathbb{P}}

\DeclareMathOperator{\conv}{conv}                                             
\DeclareMathOperator{\lspan}{span}                                             
\DeclareMathOperator{\supp}{supp}                                             
\DeclareMathOperator{\loss}{loss}                                             

\let\emptyset\varnothing

\begin{document}

\title{Tight list replicability bounds via a novel sphere covering theorem}

\author{
    Ari Blondal\thanks{McGill University, \texttt{ari.blondal@mail.mcgill.ca, hamed.hatami@mcgill.ca}. Hamed Hatami is supported by an NSERC grant.} \and Hamed Hatami\footnotemark[1] \and
    Pooya Hatami~\thanks{Ohio State University, \texttt{\{hatami.2, lalov.1, tretiak.2\}@osu.edu}} \and  Chavdar Lalov\footnotemark[2] \and Sivan Tretiak\footnotemark[2]
}

\maketitle

\begin{abstract} 
In recent years, list replicability has emerged as a framework for formalizing reproducibility in learning theory. A central question is how the required list size relates to the accuracy parameter and natural complexity measures of the hypothesis class. 

To achieve sharp bounds on list replicability, we prove a novel topological sphere covering theorem, derived from the Borsuk-Ulam theorem. 
Specifically, if the $d$-sphere is covered by open sets, each of which lies in an open hemisphere, then $d+1$ of these sets must have a common intersection.
Using this result, we obtain a sharp bound on the relationship between list size and accuracy for VC classes.
We also show that for large-margin half-spaces, provided the margin is not too large, the optimal list size equals the ambient dimension.
However, when the margin is taken to be very large, we devise a replicable algorithm achieving the minimal list size of $\lceil d/2 \rceil + 1$.
\end{abstract}

\section{Introduction}
\label{sec:intro}
Randomized learning algorithms can produce different hypotheses across multiple executions, even when trained on data drawn from the same distribution. In many contexts, it is desirable for the algorithm to produce consistent outputs across such runs. This idea is also connected to broader discussions of scientific reproducibility, in which repeated experiments are expected to yield consistent conclusions. Motivated by this perspective, a growing body of work in learning theory has introduced formal notions describing when and how randomized algorithms can be made replicable~\cite{BLM20,malliaris2022unstable,chase2023replicabilitystabilitylearning,dixon2023,bun2023stability,karbasi2023replicability,esfandiari2023replicable,Esfandiarietal23,moran2023bayesian,eaton2024replicable,kalavasis2024replicable,kalavasis2023statistical}.

One such notion is \emph{list replicability}, introduced in~\cite{chase2023replicabilitystabilitylearning,dixon2023}. Informally, the list replicability number of a concept class is the smallest integer $L$
for which there exists a learning algorithm whose set of likely output hypotheses has size at most  $L$ under every data distribution. This quantity also characterizes the amount of shared randomness required for replicable learning~\cite{ILPS22,hopkins2025}, and has found applications in differentially private learning~\cite{Alon_22_private_and_online,ghazi_user-level_2021}.

Although the definition of list replicability is purely algorithmic, there has been widespread success in reformulating it as a topological property of the space of distributions that are realizable in the given learning task
\cite{chase2023replicabilitystabilitylearning,dixon2023,localborsukulam,chornomaz2025spherical,blondal2025largemargin,blondal2025simplicial}. This perspective has enabled the use of tools from algebraic topology to derive general bounds on list size. In particular, \cite{localborsukulam} utilized a local version of the Borsuk-Ulam theorem as a powerful tool to lower-bound list replicability. They showed that for any finite cover of the $d$-sphere by antipodal-free open sets, there are at least $\lceil d / 2 \rceil + 1$ sets with a common nonempty intersection. This sphere covering result was used to show that the list size of any list-replicable learner for a concept class of VC dimension $d$ is at least $\lceil d/2\rceil+1 $, independent of the accuracy parameter $\epsilon\in (0,1/2)$. Furthermore, \cite{blondal2025largemargin} applied the sphere covering result to the class of large-margin half-spaces, showing that its list-replicability number lies between $\lceil d / 2 \rceil + 1$ and $d$.

In the local Borsuk–Ulam theorem of \cite{localborsukulam}, the factor of $1/2$ is inherent and cannot be improved. In contrast, in its application to list-replicability, this factor is an artifact of the proof technique. In particular, the lower bound of $\lceil d / 2 \rceil + 1$ on the list size is not tight: when the accuracy parameter satisfies $\epsilon < 1/d$, the list size must be at least $d$ \cite{chase2023replicabilitystabilitylearning}. Together with the persistent gap between known lower and upper bounds for the list replicability of large-margin halfspaces, this indicates that the appearance of the $1/2$ factor is due to limitations of the topological tools used so far, rather than an inherent barrier.

We confirm this hypothesis by introducing a novel sphere covering theorem. Instead of merely assuming antipodal-freeness, we impose the stronger condition that each open set in the cover of the $d$-sphere is contained in an open hemisphere. Under this assumption, we show that at least $d+1$ sets must have a common nonempty intersection. This result yields both sharper and more general bounds for list replicability across several settings.

Using our new sphere covering theorem, we prove that the list size of any list-replicable learner for a concept class of VC dimension $d$ is at least $d$ for any $\epsilon < 1/2$. Previously, this bound was only known for $\epsilon < 1/d$. As a consequence, the optimal list size for the concept class $\set{\pm1}^d$ is exactly $d$, independent of the accuracy parameter $\epsilon$. 

In the setting of large-margin half-spaces, we apply our sphere covering result to show that for margins $\gamma < 1 / \sqrt{2}$, the list replicability number is exactly $d$. In contrast, when $\gamma$ is very close to $1$, we construct an explicit learner showing that the list replicability number is $\lceil d / 2 \rceil + 1$. Together, these results show that the lower and upper bounds from \cite{blondal2025largemargin} are tight in the appropriate parameter regimes.

Finally, we study list replicability for large-margin half-spaces where the learner is restricted to outputting linear classifiers. For such algorithms, we show that the optimal list replicability is $d$, again matching the upper bound in \cite{blondal2025largemargin}.

\paragraph{Paper organization}

In \cref{sec:prelims} we collect the fundamental definitions broadly necessary for our main results. Definitions and theorems useful beyond this context are provided in their respective sections.

Our main results are contained in \cref{sec:main_results}. We begin by situating our sphere covering result, \cref{thm:main_topological_result}, within the broader framework of topological methods for list replicability. With this result in hand, we turn to our improved lower bound on $\epsilon$-list replicability of VC classes in \cref{thm:VC_bound}. Afterwards, we present our new lower and upper bounds on the list replicability number of large-margin half-spaces in  \cref{thm:margin_LR_lower_bound} and \cref{thm:very_large_margin}. Finally, an application of our topological result yields a tight lower bound on the list replicability number of linear classifiers in \cref{thm:linear_classifiers}.

As the proofs for \cref{thm:VC_bound,thm:margin_LR_lower_bound,thm:linear_classifiers} share many similarities, we have extracted those similarities into a framework theorem in \cref{sec:framework}, and deferred the full proofs to \cref{sec:VC_bound,sec:margin_LR_lower_bound,,sec:linear_classifiers} respectively.

\subsection{Preliminaries} \label{sec:prelims}

\paragraph{Partial concept classes.}
The framework of partial concept classes, introduced by Alon, Hanneke, Holzman, and Moran~\cite{alon2022theory}, extends classical learning theory to settings in which the data is guaranteed to satisfy additional structural assumptions that enable efficient learning. In particular, partial concept classes provide a natural framework for our study of large-margin half-spaces.

A \emph{partial concept class} over an arbitrary domain $\cX$ is a set 
$\cC \subseteq \{\pm 1, \star\}^{\cX}$, where each function $c \in \cC$ is called a \emph{partial concept}. We say that
a partial concept $c$ is \textit{undefined} at $x$ whenever $c(x)$ takes the star value $\star$. The \emph{support} of a partial concept $c$ is defined to be $\supp(c)\coloneq\set{x\in \cX : c(x)\neq \star}$. Note that a \emph{total} concept class $\cC\subseteq\set{\pm 1}^{\cX}$ is a special case of a partial concept class. 

\paragraph{PAC learning.}
In the \emph{probably approximately correct (PAC) learning} framework, the learner observes $n$ independent labeled examples
$S = ((x_1, y_1), \dots, (x_n, y_n))
$ sampled from an unknown but fixed distribution $\mu$ over $\mathcal{X} \times \{\pm 1\}$. The objective is to produce a \emph{hypothesis}
$h \colon \mathcal{X} \to \{\pm 1\}$
that, with high probability, performs well on the overall distribution, as measured by the \emph{population loss}
\[
\loss_\mu(h) \coloneqq \Pr_{(x, y) \sim \mu}[h(x) \neq y].\]

Note that $\mu$ is a distribution on $\cX \times \{\pm 1\}$, and therefore the learner never observes labels equal to $\star$. Moreover, the hypothesis produced by the learner should not contain $\star$ labels, since any such label would automatically be counted as an error.

We define a \emph{learning rule} to be a (possibly randomized) function $\cA$ that maps any sample
$S \in \bigcup_{n=0}^\infty (\cX \times \Set{\pm 1})^n$ 
to a hypothesis $\cA(S) \in \Set{\pm 1}^{\cX}$. Since our primary focus is learnability rather than computational efficiency, we impose no computability constraints on $\cA$. 

A distribution $\mu$ over $\cX\times\set{\pm1}$ is \emph{realizable} by a partial concept class $\cC$ if for every $n\in\mathbb{N}$, a random sample $S = ((x_i, y_i))_{i=1}^n \sim \mu^n$ is almost surely realizable by some $c \in \cC$, that is, $c(x_i) = y_i$ for all $i=1,\ldots,n$.

A partial concept $\cC$ is \emph{PAC learnable} by a learning rule $\cA$ if for any $\epsilon,\delta>0$, there exists a \emph{sample complexity} $n\coloneq n(\epsilon,\delta)$ such that for any realizable distribution $\mu$, we have 
\[
    \Pr_{S\sim\mu^n}
    [\loss_{\mu}(\cA(S)) \leq \epsilon]
    \geq 1-\delta.
\]

\paragraph{VC dimension.}
Let $A\subseteq \cX$ be a subset and $\cC$ be a partial concept class over $\cX$.  We say $\cC$ \emph{shatters} $A$ if
$\set{\pm 1}^A\subseteq \set{c|_A :~ c\in \cC}$. The VC dimension of $\cC$ is defined by
\[\VCdim(\cC)\coloneq \sup \set{ |A| :~A \subseteq \cX \text{ is shattered by $\cC$}}.\]
As in the total setting, it is shown in \cite[Theorem 3]{alon2022theory} that the finiteness of the VC dimension exactly characterizes PAC learnability for a partial concept class $\cC$. In fact, there exists a learning rule with sample complexity $O_{\epsilon,\delta}(\VCdim(\cC))$.

\paragraph{Large-margin half-spaces.}
One of the most prominent examples of a partial concept class is that of large-margin half-spaces. In this setting, we use undefined labels to formalize the large-margin assumption: data points are guaranteed to be well separated.  

In the large-margin half-spaces learning problem, the domain is the unit sphere $\bbS^{d-1} \subset \mathbb{R}^d$. For each $w\in\mathbb{S}^{d-1}$, we define a partial concept
\begin{align*}
     c_w(x) \coloneq \begin{cases}
        \sign(\langle w, x \rangle) &\text{if } |\langle w, x \rangle| \geq \gamma\\
        \star &\text{otherwise}
    \end{cases},
\end{align*}
where $\gamma>0$ is the margin parameter. Note that $c_w$ deems every $x$ within distance $\gamma$ from the hyperplane defined by $w$ as undefined.
Otherwise, it assigns $\pm 1$ depending on which side of the hyperplane $x$ is on. We denote by $\cH^d_{\gamma}$ the set  $\set{c_w : w \in \bbS^{d-1}}$ of all $\gamma$-margin partial concepts.
It is well known from the analysis of the Perceptron algorithm \cite{MR10388,rosenblatt1958perceptron} (see also \cite[Theorem 9.1]{shalev2014understanding} and \cite[Proposition 17]{alon2022theory}), that the \emph{Littlestone dimension} of $\cH^d_{\gamma}$ is bounded by $1/\gamma^2$.
Since the Littlestone dimension is a relaxation of the VC dimension, we also have 
\[\VCdim(\cH_{\gamma}^d)\leq \frac{1}{\gamma^2}\]
which is independent of the dimension $d$.
This implies that under the large-margin assumption $\gamma>0$, linear classification is efficiently PAC learnable even in high dimensions.

\paragraph{List replicability.}
List replicability was introduced in~\cite{chase2023replicabilitystabilitylearning,dixon2023} as a framework for studying the replicability of a learning problem while preserving the guarantees of PAC learning.

\begin{definition}[List replicability]\label{def:list}
A learning rule $\cA$ is an
$(\epsilon,L)$-\emph{list replicable learner} for a partial concept class $\cC$ if for every $\delta>0$ there exists a sample complexity $n \coloneq n(\delta)$ such that the following holds.
For every distribution $\mu$ realizable by $\cC$, there exists a list of hypothesis $h_1,\ldots,h_L \in \Set{\pm 1}^\cX$ such that 
\[
    \loss_\mu(h_i) \le \epsilon ~\forall i
    ~\text{ and } ~
    \Pr_{S \sim \mu^n} [\cA(S) \in \Set{h_1,\ldots,h_L}] \geq  1-\delta.
\]
The $\epsilon$-\emph{list replicability number} of $\cC$ is
\[
    \LR(\cC,\epsilon) \coloneqq \min\{L : \exists (\epsilon,L)\text{-list replicable learner for }\cC\},
\]
with $\LR(\cC,\epsilon)=\infty$ if none exists.
The \emph{list replicability number} of $\cC$ is 
\[
    \LR(\cC) \coloneqq \sup_{\epsilon>0} \LR(\cC,\epsilon).
\]
We say $\cC$ is \emph{list replicable} if $\LR(\cC)<\infty$.
\end{definition}

The goal in list replicability is to produce a PAC learner with the smallest possible list size.

\paragraph{Space of realizable distributions.} Given a partial concept class $\cC$, we define the set 
\begin{equation*}
\Delta_{\cC}\coloneq \set{\mu : \mu\;\text{realizable by}\; \cC }    
\end{equation*}
of all realizable distributions. When equipped with \emph{the total variation} (\text{TV}) \emph{distance}, $\Delta_{\cC}$
 forms a metric space, which we refer to as \emph{the space of realizable distributions} associated with $\cC$.

\paragraph{Topological terminology.}
Let $\mathbb{S}^d\subset \mathbb{R}^{d+1}$ denote the $d$-dimensional unit sphere. We say that a set $U\subseteq \mathbb{S}^d$ is \emph{antipodal-free} if it does not contain both a point $x$ and its antipode $-x$. We say a cover $\cU=\set{U_i}_{i\in I}$ of $\mathbb{S}^d$ is antipodal-free if $U_i$ is antipodal-free for all $i\in I$.

The \emph{overlap degree} of a cover $\cU=\set{U_i}_{i\in I}$ is the largest integer $k$ for which there exists $k$ sets $U_1, \dots, U_k\in \cU$ such that $\bigcap_{i=1}^k U_i\neq \emptyset$.

If $\cU=\set{U_i}_{i \in I}$ is an open cover, then a \emph{partition of unity} subordinate to $\cU$ is a collection of continuous maps $f_i \colon \bbS^d \to [0, 1]$ such that the support of $f_i$ is contained in $U_i$ for each $i \in I$, and $f_1 + \dots + f_m = 1$.

Any point $p \in \bbS^d$ defines a unique orthogonal homogeneous hyperplane in $\R^{d+1}$ given by
\[
    P = \set{x \in \R^{d+1} : \inp{p}{x} = 0}.
\]
For each point $w \neq p$ in $\bbS^d$, there is a unique line through $p$ and $w$, which intersects the plane $P$ in a unique point $w'$. The \emph{stereographic projection} through $p$ is the injective map $\pi \colon \bbS^d \setminus \set{p} \to P \cong \R^d$ given by taking $w \in \bbS^d$ to $w' \in P$.

\section{Main Results} \label{sec:main_results}

\subsection{Lower-bounding list replicability via topology}

The most well-established strategy to lower-bound the $\epsilon$-list replicability number $\LR(\cC,\epsilon)$ of a class $\cC$ hinges on exploiting the topological structure of its space of realizable distributions $\Delta_{\cC}$.

This method was independently introduced by Chase, Moran, and Yehudayoff \cite{chase2023replicabilitystabilitylearning} and by Dixon, Pavan, Woude, and Vinodchandarn \cite{dixon2023}, and was subsequently refined and extended in \cite{localborsukulam,vanderwoude2024,blondal2025largemargin,chornomaz2025spherical,blondal2025simplicial} to address a broader range of problems. These works employ a variety of topological results, such as the Poincaré-Miranda theorem \cite{chase2023replicabilitystabilitylearning}, KKM/Sperner's Lemma \cite{dixon2023,vanderwoude2024}, a Local Borsuk-Ulam theorem \cite{localborsukulam,blondal2025largemargin,chornomaz2025spherical}, and Lebesgue’s covering theorem \cite{blondal2025simplicial}, to obtain lower bounds on the list replicability number of a variety of concept classes.

In this vein, our first contribution is a sphere-covering result on the overlap degree of open covers with small diameters.  

\begin{restatable}{theorem}{butheorem}
\label{thm:main_topological_result}
    Let $A_1, \dots, A_m$ be a finite open cover of the $d$-sphere $\bbS^d$, where each $A_i$ is contained in an open hemisphere.
    Then $d+1$ of these sets have a common nonempty intersection.
\end{restatable}

We refer the reader to \cref{lemma:framework} for a precise framework detailing the use of \cref{thm:main_topological_result} in proving lower bounds for list replicability.

\subsection{List replicability of finite VC classes}
In \cite{localborsukulam}, Chase, Chornomaz, Moran, and Yehudayoff asked whether concept classes with finite VC dimension can be learned in a list-replicable way with a small list size when the accuracy parameter $\epsilon$ is large. The following bounds are known:

\begin{theorem}[\cite{localborsukulam,chase2023replicabilitystabilitylearning}]
    Let $\cC$ be a concept class with $\VCdim$ dimension $d$.
    Then
    \begin{align*}
        \text{for $\epsilon < \frac{1}{2}$,}~
        \LR(\cC, \epsilon) &\geq \frac{d}{2} + 1\\
        \text{whereas for $\epsilon < \frac{1}{d}$,}~ \LR(\cC, \epsilon) &\geq d.
    \end{align*}
\end{theorem}

Our first application of \cref{thm:main_topological_result} is to show that $d$ is the correct lower bound for the full range $\epsilon \in [0,1/2)$.

\begin{restatable}{atheorem}{theoremA}\label{thm:VC_bound}
    Let $\cC$ be a (partial) concept class with $\VCdim$ dimension $d$.
    \begin{equation*}
        \text{For any $\epsilon < \frac{1}{2}$,}~
        \LR(\cC, \epsilon) \geq d.
    \end{equation*}
\end{restatable}

Note that for the binary cube $\cC = \{\pm 1\}^d$, we have $\VCdim(\cC) = d$ and $\LR(\cC, \epsilon) \le d$ for all $\epsilon < 1/2$ \cite{chase2023replicabilitystabilitylearning}. Consequently, \cref{thm:VC_bound} implies that the $\epsilon$-list replicability number of the binary cube equals $d$ for all $\epsilon < 1/2$. In other words, increasing the accuracy parameter $\epsilon$ does not yield any improvement in list replicability for some finite VC classes.

\subsection{List replicability of large-margin half-spaces}

The list replicability number of a concept class provides bounds on various learning theory and communication complexity parameters such as VC dimension \cite{chase2023replicabilitystabilitylearning,localborsukulam}, Littlestone dimension \cite{Alon_22_private_and_online,ghazi2021sample}\footnote{The bounds in \cite{Alon_22_private_and_online,ghazi2021sample} are not explicitly stated but can be found in \cite{chase2023replicabilitystabilitylearning}.}, and sign-rank \cite{blondal2025largemargin,blondal2025simplicial}. In particular, several previously open questions in these areas were resolved by studying the list replicability number of large-margin half-spaces \cite{blondal2025largemargin}. The explicit bounds on the list replicability number of the class depended only on the underlying dimension $d$ of the problem:

\begin{theorem}[\cite{blondal2025largemargin}]
\label{thm:lower_upper_large}
    For any fixed dimension $d>1$, margin $\gamma \in (0,1)$,  and accuracy parameter $\epsilon \in (0,1/2)$,
\[
\left \lceil\frac{d}{2}\right \rceil+1  \leq  \LR(\cH_\gamma^d,\epsilon)  \leq  d.
\]
Hence, $\left \lceil\frac{d}{2}\right \rceil+1 \leq \LR(\cH_\gamma^d) \leq d$.
\end{theorem}

These general bounds left open the question of precisely how the margin $\gamma$ and accuracy parameter $\epsilon$ affect the $\epsilon$-list replicability number, if at all. We give a partial answer to this question. Our second application of \cref{thm:main_topological_result} improves the lower bounds for $\epsilon$-list replicability of large-margin half-spaces for a wide range of margins $\gamma$. This determines the exact value of the list replicability number in this range.

\begin{restatable}{atheorem}{marginlowerbound}
\label{thm:margin_LR_lower_bound}
    For any  dimension $d > 1$ and margin $\gamma \in (0, \frac{1}{\sqrt{2}})$, there exists an accuracy parameter $\epsilon \coloneqq \epsilon(\gamma, d)$ such that
    \[
        \LR(\cH^d_\gamma, \epsilon) \geq d.
    \]
    Hence, $\LR(\cH_\gamma^d)=d$ for all $\gamma\in (0,1/\sqrt{2})$.
\end{restatable}

As a complement to this result, we examine the setting in which the margin $\gamma$ is relatively large. In this setting, we prove a tight upper bound on the $\epsilon$-list replicability number.

\begin{restatable}{atheorem}{marginupperbound}
\label{thm:very_large_margin}
  For any dimension $d > 1$ and accuracy parameter $\epsilon\in [0,1/2)$, there exists some $\gamma_0\coloneq \gamma_0(d) \in(0,1)$ such that for any margin $\gamma\in (\gamma_0,1)$ we have
    \[
        \LR(\cH^d_\gamma, \epsilon) \leq  \left \lceil\frac{d}{2}\right \rceil+1.
    \]
    Hence, $\LR(\cH_\gamma^d)=\lceil \frac{d}{2}\rceil+1$ for all $\gamma\in (\gamma_0(d),1)$.
\end{restatable}
This proof uses an explicit learning rule which exploits a particular cover originally given by Chase, Chornomaz, Moran, and Yehudayoff (\cite{localborsukulam}, see \cref{thm:cover_small_overlap} for a description), in conjunction with the Lebesgue Number Lemma (\cref{thm:lebesg_number}).

\subsection*{Linear classifiers}

Our third and final application of \cref{thm:main_topological_result} addresses the question of list-replicably learning the large-margin half-space problem with linear classifiers.
A hypothesis on the domain $\bbS^{d-1}$ is a \emph{linear classifier} if it is of the form
\begin{align*}
     h_w(x) \coloneq \begin{cases}
        1 &\text{if } \langle w, x \rangle \geq 0,\\
        -1 &\text{if } \langle w, x \rangle < 0
    \end{cases}
\end{align*}
for some $w \in \bbS^{d-1}$.

The upper bound $\LR(\cH^d_\gamma, \epsilon) \leq d$ from \cref{thm:lower_upper_large} uses a learning rule which always outputs a linear classifier. We show that, under that restriction, a list of size $d$ is best possible.

\begin{restatable}{atheorem}{linearclassifiers}
\label{thm:linear_classifiers}
    Fix a dimension $d > 1$ and a margin $\gamma \in (0,1)$. For any error parameter $\epsilon \in (0,1/2)$ and list length $L > 1$, if $\cA$ is an $(\epsilon,L)$-list replicable learner for $\cH^d_\gamma$ which outputs linear classifiers, then $L \geq d$.
\end{restatable}

\subsection{Open Questions}\label{section:concluding}
In \cref{thm:VC_bound}, we remove the optimal lower bound's dependency on $\epsilon$, while we get a new dependency on $\epsilon$ in the lower bound of $\LR(\cH_\gamma^d, \epsilon)$.
We wonder if list replicability varies with $\epsilon$.

\begin{question}
    Is there a total or partial concept class $\cC$ and two error parameters $\epsilon_1, \epsilon_2 \in (0, \frac{1}{2})$ with $\epsilon_1 < \epsilon_2$ such that
    \[
        \LR(\cC, \epsilon_1) > \LR(\cC, \epsilon_2)?
    \]
\end{question}

While we have solved $\LR(\cH_\gamma^d)$ for many of the most applicable regimes, the behaviour of the list replicability number is merely bounded in other regimes.
Looking at \cref{fig:large_margin_regimes}, we ask what the transition between the top and bottom regimes looks like.

\begin{question}
    For which parameters can the bound $\LR(\cH_\gamma^d, \epsilon) = d$ be established?
\end{question}

\begin{question}
    Are there values of $\gamma \in (0, 1), \epsilon \in (0, \frac{1}{2})$ such that $\LR(\cH_\gamma^d, \epsilon)$ takes on every integer value in the range $[\lceil \frac{d}{2} \rceil + 1, d]$?
\end{question}

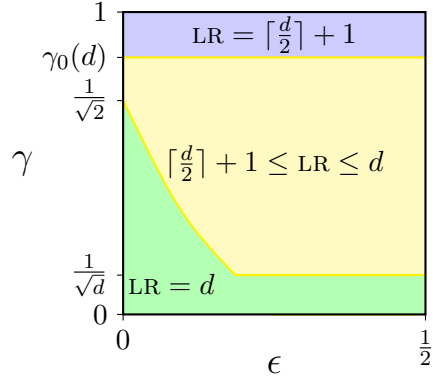
\begin{figure}[H]
    \begin{center}
        \begin{tikzpicture}[xscale=8, yscale=4]
            \def\divideroottwo{0.7071}
            \def\gammazero{0.85}
            \def\rootd{0.13}
            \def\xcurve{0.1}
            \def\ycurve{0.3}
            
            \fill[blue!20] (0,\gammazero) rectangle (0.5,1);
            \node[align=center] at (0.25, 0.925) {$\LR = \lceil \frac{d}{2} \rceil + 1$};
            
            \fill[yellow!30] (0,0) rectangle (0.5,\gammazero);
            \node[align=center] at (0.25, 0.5) {$\lceil \frac{d}{2} \rceil + 1 \leq \LR \leq d$};
            
            \fill[green!30] (0,0) -- (0,\divideroottwo) .. controls (\xcurve,\ycurve) .. (0.25,0) -- (0,0);

            \draw[yellow!100, thick] (0,\gammazero) -- (0.5,\gammazero);
            \draw[yellow!100, thick] (0,\divideroottwo) .. controls (\xcurve,\ycurve) .. (0.25,0) -- (0.5,0);

            \fill[green!30] (0,0) -- (0,\rootd) -- (0.5,\rootd) -- (0.5,0);
            \node at (0.08, 0.1) {$\LR = d$};

            \draw[yellow!100, thick] (0.18,0.1415) -- (0.185,\rootd) -- (0.5,\rootd);
        
            \draw[thick] (0,0) rectangle (0.5,1);

            \draw (-0.01,0) -- (0,0);
            \draw (-0.01,\rootd) -- (0,\rootd);
            \draw (-0.01,\divideroottwo) -- (0,\divideroottwo);
            \draw (-0.01,\gammazero) -- (0,\gammazero);
            \draw (-0.01,1) -- (0,1);

            \draw (0,-0.02) -- (0,0);
            \draw (0.5,-0.02) -- (0.5,0);

            \node[left=30pt] at (0,0.5) {\Large $\gamma$};
            \node[left=2pt] at (0,0) {$0$};
            \node[left=2pt] at (0,\rootd) {$\frac{1}{\sqrt{d}}$};
            \node[left=2pt] at (0,\divideroottwo) {$\frac{1}{\sqrt{2}}$};
            \node[left=2pt] at (0,\gammazero) {$\gamma_0(d)$};
            \node[left=2pt] at (0,1) {$1$};
            \node[below=12pt] at (0.25,0) {\Large $\epsilon$};
            \node[below=2pt] at (0,0) {$0$};
            \node[below=2pt] at (0.5,0) {$\frac{1}{2}$};
        \end{tikzpicture}
    \end{center}
    \vspace*{-5mm}
    \caption{Regimes for $\LR(\cH_\gamma^d, \epsilon)$ when $d>1$.
    When $\gamma < \frac{1}{\sqrt{d}}$, it is well known that $\VCdim(\cH_\gamma^d) = d$, so we know from \cref{thm:VC_bound} that $\LR = d$.}
    \label{fig:large_margin_regimes}
\end{figure}\label{fig:replicability_regimes}

\section{Proof of \cref{thm:main_topological_result}}

In this section, we prove our main topological result.

\butheorem*

We will use the following classical form of the Borsuk-Ulam theorem.

\begin{theorem}[{\cite[Theorem 2.1.1]{matousek2003borsuk}}]\label{thm:Borsuk-Ulam_matousek}
For every continuous mapping $f \colon \mathbb{S}^d\rightarrow\mathbb{R}^d$, there exists a point $x\in\mathbb{S}^d$ with $f(x)=f(-x)$. 
\end{theorem}

\begin{proof}[Proof of \cref{thm:main_topological_result}:]
    Let $\cA = \set{A_1, \dots, A_m}$ be a finite open cover of $\bbS^d$ such that there exist $u_1, \dots, u_m \in \bbS^d$ with
    \[
        A_i \subseteq \{x \in \bbS^d: \langle x, u_i \rangle > 0\}.
    \]
    
    Since $\bbS^d$ is a compact Hausdorff space, there is a partition of unity subordinate to $\cA$ \cite[Theorem 2.13]{rudin1987real}.
    That is, there exist continuous maps $f_i \colon \bbS^d \to [0, 1]$ such that the support of $f_i$ is contained in $A_i$
    for each $i$, and
    $f_1 + \dots + f_m = 1$.

    Consider the continuous map $g \colon \bbS^d \to \R^{d+1}$ given by
    \[
        g(x) \coloneq \sum_{i=1}^m f_i(x) u_i.
    \]
    Note that $\langle x, g(x)\rangle > 0$ for all $x$, and in particular $g$ is never zero. It follows that $h(x) \coloneq g(x)/\norm{g(x)}$ is a well-defined, continuous map from $\bbS^d$ to $\bbS^d \subset \R^{d+1}$. Moreover, $h$ inherits the property that $\langle x, h(x) \rangle > 0$ for all $x$.
    
    We claim that $h$ must be surjective.
    Indeed, assuming otherwise, let $p \in \bbS^d$ be a point outside the image of $h$, and let $\pi \colon \bbS^d \to \R^d$ be the stereographic projection through $p$. By \cref{thm:Borsuk-Ulam_matousek}, the continuous map $\pi \circ h \colon \bbS^d \to \R^d$ must identify some antipodal pair $w$ and $-w$.
    Since stereographic projection is bijective, it follows that $h(w) = h(-w)$. This contradicts the fact that $\langle x, h(x) \rangle > 0$ for all $x$.

    To complete the proof, we will show that the surjectivity of $h$ guarantees the overlap of at least $d+1$ sets of $\cA$. The image of $g$ is a subset of the space $T$, composed of the union of a finite number of convex sets.
    \begin{align*}
        T &\coloneqq \bigcup ~\set{\sigma_S : \text{for all}~S \subset [m] ~\text{where}~ \bigcap_{i \in S} A_i \neq 0}\\
        \text{where}~ \sigma_S &\coloneqq \conv\set{u_i: i \in S}
    \end{align*}
    Since we only consider $S$ where the corresponding open sets overlap, for each $S$ there is some $x_S$ with $\langle x_S, u_i \rangle > 0$ for all $i$ in $S$.
    This means that $\sigma_S$ doesn't intersect $0$, and its radial projection onto the sphere lies in a vector space of dimension at most $|S|$:
    \begin{align*}
        \Set{\frac{x}{|x|}: x \in \sigma_S} \subseteq 
        \bbS^d \cap \lspan\set{u_i : i \in S}.
    \end{align*}
    Finally, if this vector space has dimension less than $d+1$, its intersection with $\bbS^d$ has measure $0$ in $\bbS^d$.
    Therefore, if no $d+1$ sets intersect, then the image of $h$ is contained in a finite union of sets with $0$ measure, which itself has $0$ measure, contradicting its surjectivity.
\end{proof}

\section{The topological method in list replicability} \label{sec:framework}

In this section, we describe a framework for applying \cref{thm:main_topological_result} to obtain lower bounds on list replicability. This framework is applied to prove \cref{thm:VC_bound,thm:margin_LR_lower_bound,thm:linear_classifiers}, so we explain it in detail, with its applications deferred to \cref{sec:VC_bound,sec:margin_LR_lower_bound,sec:linear_classifiers} respectively.

The overarching strategy begins by identifying a subset of realizable distributions, $\Delta$, that, when endowed with the total variation metric, is homeomorphic to a high-dimensional sphere or ball. The existence of an $(\epsilon,L)$-list replicable algorithm $\cA$ then induces an open cover $\cU$ of this subset whose overlap degree lower bounds the list size $L$. If $\cU$ satisfies additional structural constraints, such as bounds on the size of the open sets or antipodal freeness, then we invoke a suitable sphere covering theorem to show that any open cover satisfying these properties must exhibit a large overlap, thereby yielding the desired lower bound.

For our purposes, \cref{thm:main_topological_result} is the appropriate sphere covering theorem. Our specific realization of the strategy above is described in the following lemma.

\begin{lemma}\label{lemma:framework}
    Let $\cA$ be an $(\epsilon,L)$-list replicable learning algorithm for a concept class $\cC \subseteq \set{\pm1,\star}^\cX$ with sample complexity $n \coloneqq n(\delta)$. Fix a dimension $d \geq 1$ and a confidence parameter $\delta \leq \frac{1}{2d}$, and suppose that, for some subset of realizable distributions $\Delta \subseteq \Delta_\cC$, there is a continuous map $\varphi \colon \bbS^{d-1} \to \Delta$.
    
    If there is an $\epsilon_0 > \epsilon$ such that, for every output $h$ of $\cA$, the set
    \[
        U_h \coloneqq \Set{
        w \in \bbS^{d-1} ~:~
        \Pr_{S \sim \varphi(w)^n}[\cA(S) = h] > \frac{1-2\delta}{L}
        ~\text{and}~
        \loss_{\varphi(w)}(h) < \epsilon_0
        }
    \]
    is contained in an open hemisphere of $\bbS^{d-1}$, then $L \geq d$. If this holds for any $\cA$ then  $\LR(\cC,\epsilon) \geq d$.
\end{lemma}

Each set $U_h$ collects all points in $\bbS^{d-1}$ which, under $\varphi$, encode distributions for which $h$ is a ``likely'' and ``accurate'' output of $\cA$. In application the size of these sets will be controlled with $\eps$, whereas the other parameters $\delta$, $n$, and $\cA$ are less impactful.

\begin{proof}
    First, we will show that the family $\cU$ is open. Let $X$ be the set of realizable distributions $x$ for which $\loss_x(h) < \epsilon_0$. Since $\loss$ is a continuous function, $X$ is open with respect to TV distance. Similarly, $\varphi^{-1}(X)$ is in turn open because of the continuity of $\varphi$. Note that $\varphi^{-1}(X)$ is exactly those $w \in \bbS^{d-1}$ satisfying $\loss_{\varphi(w)}(h) < \epsilon_0$. An analogous argument shows that the set of $w \in \bbS^{d-1}$ satisfying $\Pr_{S \sim \varphi(w)^n} [\cA(S) = h] > \frac{1-2\delta}{L}$ is also open. Therefore, each $U_h$ is open because it is the intersection of two open sets.

    Second, we will show how the $(\epsilon,L)$-list replicability of $\cA$ guarantees that $\cU$ is a cover.
    For any element $w \in \bbS^{d-1}$, the distribution $\varphi(w)$ is realizable by the definition of $\varphi$.
    The list replicability assumption guarantees some set of hypotheses $\set{h_1, \dots, h_L}$ for which $\Pr_{S \sim\varphi(w)^n} [\cA(S) \in \set{h_1, \dots, h_L}] \geq 1-\delta$ and $\loss_{\varphi(w)}(h_i) \leq \epsilon$.
    It follows by the pigeonhole principle that there is some $h$ in $\set{h_1, \dots, h_L}$ such that $U_h$ contains $w$.
    
    Finally, we will show the desired lower bound on $\LR(\cC,\epsilon)$. By assumption, each $U_h$ is contained in an open hemisphere.
    Therefore, \cref{thm:main_topological_result} guarantees some $w \in \bbS^{d-1}$ is contained in $d$ sets of $\cU$.
    By construction, there are $d$ distinct hypotheses $h_1, \dots, h_d$ and a distribution $\varphi(w)$ such that
    \[
    \Pr_{S \sim\varphi(w)^n} [\cA(S) = h_i] > \frac{1-2\delta}{L}
    \;\,\text{for all}\;\,i \in [d].\]
    Since the events $[\cA(S) = h_i]$ are disjoint, we have that 
    \[
        1>\frac{d(1-2\delta)}{L}\geq \frac{d(1-1/d)}{L}=\frac{d-1}{L},
    \]
   since $\delta \leq 1/2d$. It follows that $L \geq d$. If this is true for arbitrary $\cA$, we conclude that $\LR(\cC,\epsilon) \geq d$ by the definition of the $\epsilon$-list replicability number.
\end{proof}

\section{Proof of \cref{thm:VC_bound}}\label{sec:VC_bound}

\theoremA*

\begin{proof}
    By definition of $\VCdim$ dimension, the concept class $\cC$ contains the class $\cP \coloneqq \set{\pm1}^{[d]}$ as a subclass. It is quick to check that, for every $\epsilon$,
    \[
        \LR(\cC, \epsilon) \geq \LR(\cP, \epsilon),
    \]
    so we restrict ourselves to the analysis of $\cP$.
    
    Arguing by way of \cref{lemma:framework}, first note that the set of realizable distributions $\Delta_{\cP}$ with the metric of total variation distance is homeomorphic to $\bbS^{d-1}$ (Here $\{e_i\}_i$ is the unit vector for the $i$th dimension in $\R^d$):
    \begin{align*}
        p : \Delta_\cP &\to \bbS^{d-1}\\
        \mu &\mapsto \frac{\sum_i e_i[\mu(i, +) - \mu(i, -)]}{\sqrt{\sum_i [\mu(i, +) - \mu(i, -)]^2}}
    \end{align*}
    Note that this map is a homeomorphism precisely because only one of $\mu(i, +), \mu(i, -)$ can be non-zero at a time.

    What is left is to check that, for any $\epsilon < \frac{1}{2}$, each set $U_h$ is contained in a hemisphere
    \begin{align*}
        \set{x \in \bbS^{d-1} : \langle x, u_h \rangle > 0}\\
        \text{where}~ u_h = (h(1), h(2), \dots, h(d))
    \end{align*}
    For a realizable distribution $\mu$ and a hypothesis $h$, the loss of $h$ on $\mu$ can be calculated as
    \begin{align*}
        \loss_\mu(h) = \frac{1}{2} \left(
            1 + \sum_i h(i) \left[
                \mu(i, -) - \mu(i, +)
            \right]
        \right).
    \end{align*}
    If $p(\mu) \in U_h$, then $\loss_\mu(h)$ is less than $\frac{1}{2}$, so
    \begin{align*}
        \sum_i h(i) \left[
                \mu(i, +) - \mu(i, -)
            \right] > 0.
    \end{align*}
    Finally, we get that
    \begin{align*}
        \langle u_h, p(\mu) \rangle 
        &= \frac{\sum_i h(i) \cdot [\mu(i, +) - \mu(i, -)]}{\sqrt{\sum_i [\mu(i, +) - \mu(i, -)]^2}} > 0. \qedhere
    \end{align*}
\end{proof}

\section{Proof of \cref{thm:margin_LR_lower_bound}}\label{sec:margin_LR_lower_bound}

\marginlowerbound*

Fix $d$. Pick any $\gamma\in(0,1/\sqrt{2})$. The lower bound follows from trying to learn the collection of realizable distributions $\Delta \coloneqq \set{\mu_w : w \in \bbS^{d-1}}$, where $\mu_w$ is the uniform distribution on
\[
    \Set{(x, c_w(x)) :  x\in \supp(c_w)}.
\]
The map $\varphi\colon \mathbb{S}^{d-1} \rightarrow \Delta$ defined by $w \mapsto \mu_w$ is a homeomorphism.

 Denote by $\nu$ the spherical measure, that is, the uniformly distributed measure on $\mathbb{S}^{d-1}$, normalized so that $\nu(\mathbb{S}^{d-1})=1$.
 For each $x\in \mathbb{S}^{d-1}$, let $(x_1,\dots,x_d)$ be its coordinates in $\mathbb{R}^{d}$.
 Define
 \[
    \epsilon \coloneq \nu(\set{
        x\in\mathbb{S}^{d-1} ~:~ x_1 \geq \gamma, ~x_2 \leq -\gamma
    })
    ~\text{and}~
    \epsilon_0 \coloneq \frac{
        \nu(\set{
            x \in \mathbb{S}^{d-1} ~\colon~
            x_1 \geq \gamma, ~x_2\leq -\gamma
        })
    }{
        \nu(\set{x\in\mathbb{S}^{d-1}\,\colon\,|x_1|\geq\gamma})
    }.
\]

It is not hard to see that both $\epsilon_0$ and $\epsilon$ are strictly positive and $\epsilon_0>\epsilon$ when $\gamma \in (0,1/\sqrt{2})$. 

Now suppose $\cA$ is an $(\epsilon,L)$-list replicable learner for $\cH_\gamma^d$. Let $T \subseteq \set{\pm1}^{\bbS^{d-1}}$ be the collection of all possible hypotheses output by $\cA$.
For each hypothesis $h \in T$, define the set $U_h$
\begin{align*}
    U_h &= \Set{
        w \in \bbS^{d-1} ~:~
        \Pr_{S \sim \mu_w^{n}} [
            \cA(S) = h
        ] > \frac{1-2\delta}{L},\,\loss_{\mu_w}(h)<\epsilon_0
    }
\end{align*}
of distributions on which $\cA$ is likely to output $h$.

\begin{claim}\label{claim:margin_halfsphere}
   For each $h\in T$, if $u,w \in U_h$, then $\langle u, w \rangle > 0$. In particular, $U_h$ is contained in an open hemisphere.
\end{claim}

\begin{proof}[Proof of Claim.]
    Let $u, w \in U_h$.
    Define the set on which $c_u$ and $c_{w}$ disagree:
    \[
        \cE_{u, w} \coloneqq
        \set{x \in \supp c_u \cap \supp c_{w} : c_u(x) \neq c_{w}(x)}.
    \]
We have
\begin{equation} 
\label{eq:error_support} 
\begin{split}
    \nu(\cE_{u, w}) &=\nu(\supp(c_w))\frac{\nu(\cE_{u,w})}{\nu(\supp(c_w))}\\
    & \leq \nu(\supp(c_w))(\loss_{\mu_u}(h) + \loss_{\mu_{w}}(h)) \\
    &< \nu(\supp(c_w)) 2\epsilon_0 \\
    &= 2\epsilon.
\end{split}
\end{equation}
The first inequality is implied by the fact that $h$ is not consistent with either $\mu_u$ or $\mu_{w}$ on each point in $\cE_{u,w}$. The second inequality follows directly from our definition of $U_h$. The last equality is a consequence of $\nu(\supp(c_w))=\nu(\set{x\in \mathbb{S}^{d-1}:|x_1|\geq \gamma})$ for any $w$ and our definitions of $\epsilon$ and $\epsilon_0$.
    
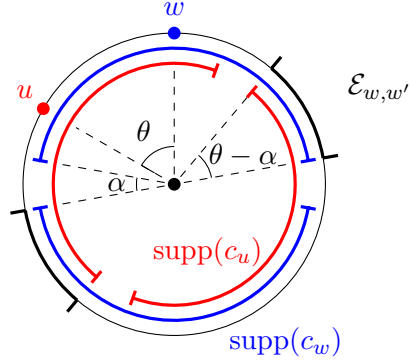
\begin{figure}[H]
    \begin{center}
        \begin{tikzpicture}[baseline=(current bounding box.center),scale=1]
            \draw (0, 0) circle (2);

            \coordinate (A) at (90:2);
            \coordinate (B) at (150:2);
            \filldraw[thick] (0:0) circle (2pt);

            \filldraw[thick, blue] (A) circle (2pt);
            \node[anchor=south, blue] at (90:2.1) {$w$};
            \node[anchor=north, blue] at (-50:2.3) {$\operatorname{supp}(c_w)$};
            \filldraw[thick, red] (B) circle (2pt);
            \node[anchor=south east, red] at (B) {$u$};
            \node[anchor=south, red] at (-70:1.3) {$\operatorname{supp}(c_u)$};

            \draw[dashed]
                (90:1.5) -- (0:0) -- (150:1.5);
            \draw
                (90:0.5) arc[start angle=90, end angle=150, radius=0.5] -- (150:0.3);
            \node[anchor=-60] at (120:0.5) {$\theta$};

            \draw[dashed]
                (170:1.5) -- (0:0)
                (190:1.5) -- (0:0);
            \draw
                (170:0.5) arc[start angle=170, end angle=190, radius=0.5] -- (190:0.5);
            \node[anchor=east] at (180:0.5) {$\alpha$};

            \draw[dashed]
                (10:1.5) -- (0:0)
                (50:1.5) -- (0:0);
            \draw
                (10:0.5) arc[start angle=10, end angle=50, radius=0.5] node[anchor=180] {$\theta - \alpha$} -- (50:0.5);

            \draw[very thick, blue]
                (10:1.8) arc[start angle=10, delta angle=160, radius=1.8] -- (170:1.8)
                (190:1.8) arc[start angle=190, delta angle=160, radius=1.8] -- (350:1.8)
                (10:1.7) -- (10:1.9)
                (170:1.7) -- (170:1.9)
                (190:1.7) -- (190:1.9)
                (350:1.7) -- (350:1.9)
                ;

            \draw[very thick, red]
                (70:1.6) arc[start angle=70, delta angle=160, radius=1.6] -- (230:1.6)
                (250:1.6) arc[start angle=250, delta angle=160, radius=1.6] -- (410:1.6)
                (70:1.5) -- (70:1.7)
                (230:1.5) -- (230:1.7)
                (250:1.5) -- (250:1.7)
                (410:1.5) -- (410:1.7)
                ;

            \draw[very thick]
                (10:2) arc[start angle=10, end angle=50, radius=2] -- (50:2)
                (190:2) arc[start angle=190, end angle=230, radius=2] -- (230:2)
                (10:2) -- (10:2.2)
                (50:2) -- (50:2.2)
                (190:2) -- (190:2.2)
                (230:2) -- (230:2.2);

            \node at (25:3) {$\cE_{w, w'}$};
        \end{tikzpicture}
    \end{center}
    \vspace*{-5mm}        
    \caption{An illustration of the disagreement set $\cE_{u, w}$ between $c_u$ and $c_w$, projected onto the plane common to $u, w$, and $0$.}
    \label{fig:angle}
\end{figure}
    
Now let $\theta \coloneqq \arccos(\langle u, w \rangle)$ be the angle between $u$ and $w$. We have
\begin{equation} 
\label{eq:error_support_equality} 
\begin{split}
\nu(\cE_{u, w}) &=\nu(\set{x \in \supp c_u \cap \supp c_{w} : c_u(x) \neq c_{w}(x)}) = \\ & = 2\nu(\set{x\in \mathbb{S}^{d-1}\,\colon \langle u,x\rangle\geq\gamma,\langle w,x\rangle\leq-\gamma})= \\ & = 2\nu(\set{x\in \mathbb{S}^{d-1}\,\colon x_1\geq\gamma,\cos(\theta)x_1+\sin(\theta)x_2\leq-\gamma})
\end{split}
\end{equation}
where the last equality follows from the rotational invariance of the spherical measure.

For the sake of contradiction, assume that $\theta \in [\pi/2,\pi]$. Next we show
\begin{equation}\label{eq:inclusion_angle}
 \nu(\set{x\in\mathbb{S}^{d-1}\,\colon x_1\geq\gamma,\,x_2\leq-\gamma})
 \subseteq
 \nu(\set{x\in\mathbb{S}^{d-1}\;\colon x_1\geq\gamma,\cos(\theta)x_1+\sin(\theta)x_2\leq-\gamma}).   
\end{equation}
Indeed, it is enough to check $\cos(\theta)x_1+\sin(\theta)x_2\leq-\gamma$ assuming $x_1\geq\gamma$ and $x_2\leq-\gamma$. Since $\theta\in[\pi/2,\pi]$, cosine is non-positive and sine is non-negative. Thus, we only need to verify $\sin(\theta)-\cos(\theta)\geq1$. Since, the left hand-side is non-negative, we square both sides to get 
\begin{equation*}
\sin^2(\theta)-2\sin(\theta)\cos(\theta)+\cos^2(\theta)\geq 1
\end{equation*}
which holds as $\theta\in[\pi/2,\pi]$. We conclude that the inclusion in \eqref{eq:inclusion_angle} holds. 

Now by \eqref{eq:error_support_equality} and \eqref{eq:inclusion_angle}, we have 
\begin{equation}\label{eq:lower_measure}
    \nu(\cE_{u,w})\geq
    2\nu(\set{x\in\mathbb{S}^{d-1}\,\colon x_1\geq\gamma,\,x_2\leq-\gamma})
    =2\epsilon.
\end{equation}
Finally, by \eqref{eq:error_support} and \eqref{eq:lower_measure}, we have $2\epsilon\leq\nu(\cE_{u,w})<2\epsilon$ which is a contradiction. Thus, if $u, w \in U_h$, then $\theta < \frac{\pi}{2}$, and so $\langle u, w \rangle > 0$.
\end{proof}
Having proved \cref{claim:margin_halfsphere}, we have that \cref{lemma:framework} implies \cref{thm:margin_LR_lower_bound}.

\begin{remark}
    Examining the relationship between $\epsilon$ and $\gamma$ in the above proof, we can see that, for any dimension $d>1$, there exists a continuous decreasing function $\epsilon\colon(0,1/\sqrt{2}) \rightarrow (0,1/4)$ such that: $\epsilon(\gamma)\rightarrow1/4$ as $\gamma \rightarrow 0^+$,  $\epsilon(\gamma)\rightarrow0$ as $\gamma \rightarrow {(1/\sqrt{2})}^-$, and $
    \LR(\cH_\gamma^d, \epsilon(\gamma))
    \geq d$ for any $\gamma \in (0,1/\sqrt{2})$. This justifies the depiction of the boundary between the lower green and middle yellow regions in \cref{fig:large_margin_regimes}.
\end{remark}

\section{Proof of \cref{thm:very_large_margin}}

\marginupperbound*

Before we begin, we introduce some basic notation. We endow $\mathbb{S}^{d-1}$ with the $\ell_2$ metric and for each $x\in\mathbb{S}^{d-1}$ and all $\eta>0$ we denote by $B(\eta,x)$ all points in $\mathbb{S}^{d-1}$ which are less than $\eta$ away from $x$. 

For each realizable distribution $\mu$, let $\supp(\mu)$ denote its support. We define
\[
    \supp^{+}(\mu)\coloneqq \set{x \in \mathbb{S}^{d-1}\colon(x,1)\in\supp(\mu)}
\]
to be the positive label support of $\mu$, and likewise for $\supp^{-}(\mu)$.

Next, we present some topological results that will be used in the proof.

\begin{theorem}[{\cite[Theorem A]{localborsukulam}}]\label{thm:cover_small_overlap}
    There exists a finite antipodal-free open cover $\cU\coloneqq \{U_i\}_{i=1}^{n}$ of $\mathbb{S}^{d-1}$ such that each point $x\in\mathbb{S}^{d-1}$ is contained in at most $\left \lceil \frac{d}{2}\right \rceil+1$ of the sets $U_i$.
\end{theorem}

We also utilize the following standard result in the theory of compact metric spaces.

\begin{theorem}[Lebesgue Number Lemma]\label{thm:lebesg_number}
Let $(X,d)$ be a compact metric space, and let $\mathcal{U}$ be an open cover of $X$. 
Then there exists a number $\eta > 0$ such that for every $x \in X$, 
the open ball $B(x,\eta)$ is contained in some $U \in \mathcal{U}$.
\end{theorem}

\begin{proof}
    Since $X$ is compact, we can assume without loss of generality that the cover $\cU$ is finite. Define the function $f(x)\coloneq \max_{i\in I}d(x,U_i^c)$. By the continuity of $f$ and compactness of $X$, it follows that $f$ must achieve its minimum. The minimum cannot be zero as $\cU$ is a cover. Taking $\eta$ to be the minimum of $f$ completes the proof.
\end{proof}

Finally, we present the proof \cref{thm:very_large_margin}.  

\begin{proof}[Proof of \cref{thm:very_large_margin}]
Pick any $\epsilon\in[0,1/2)$. 

Let $\cU$ be the antipodal-free open cover $\cU$ of $\mathbb{S}^{d-1}$ from \cref{thm:cover_small_overlap}. By \cref{thm:lebesg_number}, there exists a $\eta>0$ such that each ball $B(\eta,x)$ is contained in some $U_i$. 

Pick $\gamma(d)>0$ to be sufficiently close to $1$ so that for any realizable distribution $\mu$ there exist a point $w_{\mu}\in\mathbb{S}^{d-1}$ for which the following containments hold:
\[\supp^+(\mu)\subseteq B(\eta/2,w_{\mu})~~\text{and}~~ \supp^-(\mu)\subseteq B(\eta/2,-w_{\mu}).\]

Next we construct a $(0,\left \lceil \frac{d}{2}\right \rceil+1)$-list replicable algorithm with sample complexity $1$.
\begin{enumerate}
    \item Sample one point $(x,y)$ from $\mu$.
     Let \[u(x,y)\coloneq \begin{cases} x & \text{for}~ y=1\\
     -x & \text{for} ~y=-1\end{cases}\]
    \item Find the smallest index $i$ for which $B(\eta,u)\subseteq U_i$ for some $U_i \in \cU$. Such an $i$ is guaranteed to exist by the Lebesgue Number Lemma (\cref{thm:lebesg_number}).
    \item Output the signed indicator hypothesis
    \[h_{U_i}(x)\coloneq \begin{cases} 1 & \text{for} ~x\in U_i\\
     -1 & \text{for} ~ x\in U_i^c\end{cases}\]
\end{enumerate}
Firstly, we show accuracy. With probability $1$ we have the chains of containments 
\begin{equation*}\label{eq: containments_1}
    \supp^+(\mu)\subseteq B(\eta/2,w_{\mu})\subseteq B(\eta,u)\subseteq U_i
\end{equation*}
\begin{equation*}
    \supp^-(\mu)\subseteq B(\eta/2,-w_{\mu})\subseteq B(\eta,-u)\subseteq U_i^c
\end{equation*}
where the second containment in both chains follows from the fact that $u\in B(\eta/2,w_{\mu})$, and the third containment in the second chain follows from the property that $\cU$ is antipodal free. We conclude  $h_{U_i}(x)$ achieves zero error with probability 1. 

Finally, we check the list replicability. With probability 1 we have $B(\eta/2,w_{\mu})\subseteq U_i$ and, in particular, $w_{\mu}\in U_i$. By \cref{thm:cover_small_overlap}, there could be at most $\left \lceil \frac{d}{2}\right \rceil+1$ sets $U_j\in \cU$ containing $w_{\mu}$. Hence, with probability $1$, we output from a list of $\left \lceil \frac{d}{2}\right \rceil+1$ different hypotheses.
\end{proof}

\section{Proof of \cref{thm:linear_classifiers}}\label{sec:linear_classifiers}

\linearclassifiers*

\begin{proof}
Consider the set of realizable distributions $\Delta \coloneqq \set{\mu_w : w \in \bbS^{d-1}}$, where $\mu_w$ is the uniform distribution on
\[
    \Set{(x, c_w(x)) : x\in \supp(c_w)}.
\]
The map $\varphi\colon \mathbb{S}^{d-1} \rightarrow \Delta$ defined by $w \mapsto \mu_w$ is a homeomorphism.

By assumption $\cA$ outputs linear classifiers of the form
\begin{align*}
     h_w(x) \coloneq \begin{cases}
        1 &\text{if } \inp{w}{x} \geq 0,\\
        -1 &\text{if } \inp{w}{x} < 0
    \end{cases}
\end{align*}
for some $w \in \bbS^{d-1}$. Fix $\epsilon < \epsilon_0 < 1/2$. For any such output $h_w$, we will show that the set
\[
    U_{h_w} \coloneqq \Set{
    y \in \bbS^{d-1} ~:~
    \Pr_{S \sim \varphi(y)^n}[\cA(S) = h_w] > \frac{1-2\delta}{d-1}
    ~\text{and}~
    \loss_{\varphi(y)}(h_w) < \epsilon_0
    }
\]
is contained in the open hemisphere
\[
    H_w \coloneqq \Set{y \in \bbS^{d-1} : \inp{w}{y} > 0}
.\]
Indeed, if $\inp{w}{y} \leq 0$, then $h_w$ correctly labels no more than half of the support of $\varphi(y)$. Since $\varphi(y)$ is a uniform distribution, it follows that $\loss_{\varphi(y)}(h_w) \geq 1/2 > \epsilon_0$. We deduce that $U_{h_w} \subseteq H_w$, so applying \cref{lemma:framework} completes the proof.
\end{proof}

\bibliographystyle{alphaurl}
\bibliography{refs}

\end{document}